\documentclass[article, 11pt]{amsart}
\usepackage[margin=1in]{geometry}
\usepackage{amssymb, amsthm, amsmath, color, tikz, natbib, hyperref, enumitem}
\usetikzlibrary{arrows, shapes.geometric}
\newtheorem{theorem}{Theorem}
\newtheorem{lemma}[theorem]{Lemma}
\newtheorem{corollary}[theorem]{Corollary}
\newtheorem{proposition}[theorem]{Proposition}
\theoremstyle{remark}
\newtheorem{remark}[theorem]{Remark}
\theoremstyle{definition}
\newtheorem{definition}[theorem]{Definition}
\newcommand{\oo}{\infty}
\newcommand{\R}{\mathbb{R}}
\newcommand{\Z}{\mathbb{Z}}
\newcommand{\PHI}{\varphi}
\newcommand{\DGM}[1]{(\text{Dgm}_{#1}, w_{#1})}
\newcommand{\HP}{\mathbb{R}_<^2}
\DeclareMathOperator{\ID}{id}
\DeclareMathOperator{\C}{cost}
\begin{document}

\title{Embeddings of Persistence Diagrams into Hilbert Spaces}

\author{Peter Bubenik}
\address{Department of Mathematics, University of Florida}
\email{peter.bubenik@ufl.edu}
\urladdr{https://people.clas.ufl.edu/peterbubenik/}

\author{Alexander Wagner}
\address{Department of Mathematics, University of Florida}
\email{wagnera@ufl.edu}
\urladdr{https://people.clas.ufl.edu/wagnera/}

\maketitle

\begin{abstract}
Since persistence diagrams do not admit an inner product structure, a map into a Hilbert space is needed in order to use kernel methods. It is natural to ask if such maps necessarily distort the metric on persistence diagrams. We show that persistence diagrams with the bottleneck distance do not even admit a coarse embedding into a Hilbert space. As part of our proof, we show that any separable, bounded metric space isometrically embeds into the space of persistence diagrams with the bottleneck distance. As corollaries, we obtain the generalized roundness, negative type, and asymptotic dimension of this space.
\end{abstract}

\section{Introduction}
\label{sec:intro}

Kernel methods, such as support vector machines or principal components analysis, are machine learning algorithms that require an inner product on the data \citep{MR2450103}. When the original data set $X$ lacks an inner product or one would like a higher-dimensional representation of the data, a standard approach is to map the data into a Hilbert space $\mathcal{H}$. Such a mapping is called a feature map and kernel methods are implicitly performed in the codomain of the feature map. While specifying an explicit feature map may be difficult, it turns out to be equivalent to the often simpler task of constructing a positive definite kernel on the data. This equivalence is important for the practical success of kernel methods but should not obscure the fact that there is an underlying feature map $\PHI: X \to \mathcal{H}$ and that the associated learning algorithm works with $\PHI(X) \subseteq \mathcal{H}$. Because of this, when $X$ represents stable signatures of input data, one would like a feature map $\PHI$ that changes the original metric as little as possible. 

Persistent homology takes in a one-parameter family of topological spaces and outputs a signature, called the persistence diagram, of this family's changing homology. There is a natural metric on one-parameter families of topological spaces, called the interleaving distance, and a family of metrics on persistence diagrams, called the $p$-Wasserstein distances. When $p = \oo$,  changes in the input of persistent homology cause at most proportional changes in the output \citep{MR2279866, Chazal:2009:PPM:1542362.1542407}. This stability supports the use of persistent homology for machine learning because it guarantees that small perturbations of the data, such as those caused by measurement noise, do not cause large changes in the associated features. If one would like to apply kernel methods to persistence diagrams, a natural first question is whether the metrics on persistence diagrams can be induced by an inner product. More precisely, does there exist an isometric embedding of persistence diagrams into a Hilbert space? We show in Section \ref{subsec:isoembed} that the impossibility of such an isometric embedding follows from work of \citet{2019arXiv190301051T} and classical results of \citet{MR1503248, MR1501980}. In other words, any feature map from persistence diagrams into a Hilbert space necessarily distorts the original metric. 
 
Our main results consider the $\oo$-Wasserstein distance, also called the bottleneck distance. Among the $p$-Wasserstein distances on persistence diagrams, this is the only case for which persistent homology is $1$-Lipschitz. Isometric embeddings require distances to be exactly preserved. More general are bi-Lipschitz embeddings which are allowed to distort distances at most linearly. Considerably more general are coarse embeddings, which need not be continuous and only require that distances be distorted in a uniform, but potentially non-linear, way. Coarse embeddings are an important notion in geometric group theory and coarse geometry \citep{gromov1993asymptotic, 	roe2003lectures}. We show that the space of persistence diagrams with the bottleneck distance does not admit a coarse embedding into any Hilbert space (Theorem \ref{thm:embeddingdb}). In other words, the distortion caused by a feature map to the bottleneck distance is not uniformly controllable. In fact, even if one restricts to the subspace of (finite) persistence diagrams arising as the homology of a filtered finite simplicial complex, there still does not exist a coarse embedding of this subspace into a Hilbert space (Remark~\ref{rem:finite} and Lemma~\ref{lem:finite}). This result about distortions of embeddings is something that people working with persistence diagrams have noticed in practice. Philosophically, this is to be expected because bottleneck distance is an $\ell^{\oo}$-type distance, and $\ell^{\oo}$ can only be embedded in $\ell^2$ with distortion growing with dimension. Our paper makes such an argument rigorous. As corollaries of Theorem \ref{thm:embeddingdb}, we obtain the generalized roundness, negative type, and asymptotic dimension of persistence diagrams with the bottleneck distance (Corollary~\ref{cor:grbottleneck}, Remark~\ref{rem:negative-type} and Corollary~\ref{cor:asymptotic-dimension}). Toward our proof of Theorem \ref{thm:embeddingdb}, we show that any separable, bounded metric space isometrically embeds into the space of persistence diagrams with the bottleneck distance (Theorem \ref{thm:finitems}). Our proof of Theorem \ref{thm:embeddingdb} combines Theorem \ref{thm:finitems} with ideas of \citet{dranishnikov2002uniform} and \citet{enflo1970problem}.
  
\subsection{Related work}
\label{subsec:relatedwork}

\citet{MR3968607} have investigated bi-Lipschitz embeddings of persistence diagrams into separable Hilbert spaces. They've shown the impossibility of a bi-Lipschitz embedding into a finite-dimensional Hilbert space and that bi-Lipschitz embeddings into infinite-dimensional, separable Hilbert spaces only exist when restrictions are placed on the cardinality and spread of the persistence diagrams under consideration. \citet{2019arXiv190202288B} have shown that the space of persistence diagrams with the $p$-Wasserstein distance for $p < \oo$ has a discrete subspace that fails to have property A. The relevance of this result is that a discrete metric space with property A admits a coarse embedding into a Hilbert space \citep[Theorem 2.2]{MR1728880}. \citet[Theorem 4.37]{MR3927353} have shown that there exist cubes of arbitrary dimension with the $\ell^\oo$ distance which isometrically embed into the space of persistence diagrams with the bottleneck distance.

\section{Background}
\label{sec:background}

\subsection{The space of persistence diagrams}
\label{subsec:spaceofpd}
In this section, we define persistence diagrams and a family of associated metric spaces. Persistence diagrams naturally arise as the output of persistent homology, which describes the changing homology of a one-parameter family of topological spaces. Persistence diagrams are usually defined to be multisets. We find it convenient to instead define them as indexed sets.

\begin{definition}
Denote $\{(x,y) \in \R^2\ |\ x < y\}$ by $\HP$. A \emph{persistence diagram} is a function from a countable set $I$ to $\HP$, i.e. $D: I \to \HP$.
\end{definition}

To define the relevant metrics on persistence diagrams, we need two preliminary definitions.

\begin{definition}
Suppose $D_1: I_1 \to \HP$ and $D_2: I_2 \to \HP$ are persistence diagrams. A \emph{partial matching} between them is a triple $(I_1', I_2', f)$ such that $I_1' \subseteq I_1$, $I_2' \subseteq I_2$, and $f: I_1' \to I_2'$ is a bijection.
\end{definition}

The $p$-Wasserstein distance between two persistence diagrams will be the minimal cost of a partial matching between them. More precisely, the cost of a partial matching is the $\ell^p$ norm of the distances between matched pairs and the distances between unmatched pairs and $\Delta$, the diagonal in $\R^2$. 

\begin{definition}
Suppose $D_1: I_1 \to \HP$ and $D_2: I_2 \to \HP$ are persistence diagrams and $(I_1', I_2', f)$ is a partial matching between them. Equip $\HP$ with the metric $d_\oo(a,b) = \| a - b \|_\oo = \max(|a_x - b_x|, |a_y - b_y|)$. For $a \in \HP$, observe that $d_\oo(a, \Delta) = \inf_{t \in \Delta} d_\oo(a, t) = (a_y - a_x)/2$. The \emph{$p$-cost of $f$} is denoted $\C_p(f)$ and defined as follows. 
If $p < \oo$,
\[
\C_p(f) = \left( \sum_{i \in I_1'}d_\oo (D_1(i), D_2(f(i)))^p + \sum_{i \in I_1 \setminus I_1'}d_\oo (D_1(i), \Delta)^p + \sum_{i \in I_2 \setminus I_2'} d_\oo (D_2(i), \Delta)^p\right)^{1/p}.
 \]
 If $p = \oo$,
 \[
 \C_{\oo}(f) = \max \{ \sup_{i \in I_1'} d_\oo (D_1(i), D_2(f(i)), \sup_{i \in I_1 \setminus I_1'} d_\oo(D_1(i), \Delta), \sup_{i \in I_2 \setminus I_2'} d_\oo(D_2(i), \Delta)\}.
 \]
If any of the terms in either expression are unbounded, we define the cost to be infinity.
\end{definition}

\begin{definition}[\citealt{MR2279866, MR2594441}]
Let $1\leq p \leq \oo$. If $D_1$, $D_2$ are persistence diagrams, define 
\[
\tilde{w}_p(D_1, D_2) = \inf \{ \C_p(f)\ |\ \text{f is a partial matching between $D_1$ and $D_2$} \}.
\]
Let $\DGM{p}$ denote the metric space of persistence diagrams $D$ that satisfy $\tilde{w}_p(D,\emptyset)< \oo$ modulo the relation $D_1 \sim D_2$ if $\tilde{w}_p(D_1, D_2) = 0$, where $\emptyset$ is shorthand for the unique persistence diagram with empty indexing set. The metric $w_p$ is called the \emph{$p$-Wasserstein distance} and $w_{\oo}$ is called the \emph{bottleneck distance}.
\end{definition}

Note that the empty partial matching is the only one between $D: I \to \HP$ and $\emptyset$. Hence, $\tilde{w}_p(D, \emptyset) = (\sum_{i \in I} d_\oo (D(i), \Delta)^p)^{1/p}$.

\subsection{Negative type and kernels} 
\label{subsec:negtype}
The following definition and theorem equate the problem of defining a feature map on a set to the frequently simpler problem of defining a positive definite kernel. Theorem \ref{thm:kernelfeaturemap} and the fact that kernel methods require access to only the inner products of elements is the content of the so-called kernel trick.

\begin{definition}
Let $X$ be a nonempty set. A symmetric function $k: X \times X \to \R$ is a \emph{positive definite kernel} if for any $n \in \mathbb{N}$, $c_1, \dots, c_n \in \R$, and $x_1, \dots, x_n \in X$,
\[
\sum_{i,j = 1}^n c_ic_jk(x_i,x_j) \geq 0.
\]
\end{definition}

\begin{theorem}[{\citealt[Theorem 4.16]{MR2450103}}]
\label{thm:kernelfeaturemap}
	Let $X$ be a nonempty set. A function $k: X \times X \to \R$ is a positive definite kernel iff there exists a Hilbert space $\mathcal{H}$ and a feature map $\PHI: X \to \mathcal{H}$ such that $\langle \PHI(x), \PHI(y) \rangle = k(x,y)$ for every $x,y \in X$.
\end{theorem}

We now turn to the definition of negative type, which is closely related to positive definite kernels and to the embeddability of metric spaces into Hilbert spaces. Negative type played a central role in work of \citet{MR1503248, 	MR1501980} characterizing semi-metric spaces that admit an isometric embedding into a Hilbert space; see Theorems~\ref{thm:schoenberg} and \ref{thm:negtypeiso}. \citet{enflo1970problem} also implicitly used negative type to answer negatively the question of Smirnov on whether every separable metric space is uniformly homeomorphic to a subset of $L_2[0, 1]$. The equivalence between Enflo's notion of generalized roundness and the older notion of negative type was not proven until much later by \cite{lennard1997generalized}, giving a geometric characterization to the notion of negative type and, in particular, to the existence of isometric embeddings into Hilbert spaces. We refer the reader to \citet{MR747302} and \citet{wells2012embeddings} for a more thorough treatment of the results referenced here. We remark that what we call a semi-metric space in the following definition is called a quasi-metric space in \citet{wells2012embeddings}.

\begin{definition}
A \emph{semi-metric space} is a nonempty set $X$ together with a function $d: X \times X \to [0,\oo)$ such that $d(x,x) = 0$ and $d(x,y) = d(y,x)$ for every $x,y \in X$.
\end{definition}

\begin{definition}
	Let $q \geq 0$. A semi-metric space $(X,d)$ is said to be of \emph{$q$-negative type} if for any $n \in \mathbb{N}$,  $x_1, \dots, x_n \in X$, and $a_1, \dots, a_n \in \R$ satisfying $\sum_{i=1}^{n} a_i = 0$, the following inequality is satisfied.
	\[
	\sum_{i,j =1}^{n} a_ia_jd(x_i,x_j)^q \leq 0
	\]
	We define the negative type of a semi-metric space $(X, d)$ to be the supremum of the set of $q \in [0, \oo)$ such that $(X,d)$ is of $q$-negative type.
\end{definition}

A relationship between positive definite kernels and negative type is given in the following.

\begin{theorem}[{\citealt[Lemma 2.1, Theorem 2.2]{MR747302}}]
\label{thm:schoenberg}
Let $(X,d)$ be a semi-metric space. The following are equivalent.
\begin{enumerate}
	\item $(X,d)$ is of $1$-negative type.
	\item For any $x_0 \in X$, $k(x,y) = d(x,x_0) + d(y,x_0) - d(x,y)$ is a positive definite kernel.
	\item $k(x,y) = e^{-td(x,y)}$ is a positive definite kernel for every $t > 0$.
\end{enumerate}	
\end{theorem}

The negative type of a semi-metric space is closely related to questions regarding its embeddability into Hilbert spaces. An isometric embedding of a semi-metric space $(X, d)$ into a Hilbert space $\mathcal{H}$ is a map $\PHI: X \to \mathcal{H}$ satisfying $d(x,y) = \| \PHI(x) - \PHI(y) \|_{\mathcal{H}}$ for every $x, y \in X$.

\begin{theorem}[{\citealt[Theorem 2.4, Remark 3.2]{wells2012embeddings}}]
\label{thm:negtypeiso}
A semi-metric space admits an isometric embedding into a Hilbert space iff it is of $2$-negative type.
\end{theorem}

Besides $2$-negative type characterizing isometric embeddability into a Hilbert space, the following theorem states the important property that negative type is downward closed.

\begin{theorem}[{\citealt[Theorem 4.7]{wells2012embeddings}}]
\label{thm:negtypedownward}
	Suppose $(X,d)$ is a semi-metric space of $q$-negative type. Then it is of $q'$-negative type for any $0 \leq q' \leq q$.
\end{theorem}

\subsection{Isometric embeddability of diagram space}
\label{subsec:isoembed}

It was shown by \citet[Theorem 3.2]{2019arXiv190301051T} that $\DGM{p}$ is not of $1$-negative type for any $1 \leq p \leq \oo$. This leads to the following negative result.

\begin{theorem}
$\DGM{p}$ does not admit an isometric embedding into a Hilbert space for any $1 \leq p \leq \oo$.
\end{theorem}

\begin{proof}
	Let $1 \leq p \leq \oo$. Since $\DGM{p}$ is not of $1$-negative type, by Theorem \ref{thm:negtypedownward}, $\DGM{p}$ is not of $2$-negative type and so does not admit an isometric embedding into a Hilbert space by Theorem \ref{thm:negtypeiso}.
\end{proof}

\subsection{Coarse embeddings and related notions}
\label{subsec:coarseembedding}

If instead of demanding that distances be exactly preserved, we only require that distances be contracted or expanded a uniform amount, we arrive at the following definition.

\begin{definition}
	A map $f:(X,d) \to (Y, d')$ is a \emph{coarse embedding} or \emph{uniform embedding} if there exists non-decreasing $\rho_-, \rho_+ :[0, \oo) \to [0, \oo)$ such that
	\begin{enumerate}
		\item $\rho_-(d(x,y)) \leq d'(f(x), f(y)) \leq \rho_+(d(x,y))$ for all $x,y \in X$, and
		\item $\lim_{t \to \oo} \rho_-(t) = \oo$.
	\end{enumerate}
\end{definition}

Note that if $\rho_-(x) = Ax$ and $\rho_+(x) = Bx$ for some $0 < A \leq B$ then $f$ is a bi-Lipschitz embedding. This definition was introduced by \citet[p. 218]{gromov1993asymptotic} where he posed the question of whether every separable metric space, of which $\DGM{p}$ are examples \citep{mileyko2011probability, MR3230014}, admits a coarse embedding into a Hilbert space. This question was answered negatively by \citet{dranishnikov2002uniform}. The following definition gives a coarse analogue of covering dimension.

\begin{definition}
Let $n$ be a non-negative integer. A metric space $(X,d)$ has \emph{asymptotic dimension} $\leq n$ if for every $R > 0$ there exists a cover $\mathcal{U}$ of $X$ such that every ball of radius $R$ intersects at most $n+1$ elements of $\mathcal{U}$ and $\sup_{U \in \mathcal{U}} \sup \{ d(x,y)\ |\ x,y \in U\} < \oo$.
\end{definition}

\begin{theorem}[{\citealt[Example 11.5]{roe2003lectures}}]
\label{thm:roe}
If $X$ is a metric space with finite asymptotic dimension, then there exists a coarse embedding of $X$ into a Hilbert space.
\end{theorem}

Property A is a simple condition for discrete metric spaces that implies coarse embeddability into a Hilbert space.

\begin{definition}[{\citealt[Definition 2.1]{MR1728880}}]
	A discrete metric space $(X,d)$ has \emph{property A} if for any $r > 0$, $\varepsilon > 0$ there is a family of finite subsets $\{A_x\}_{x \in X}$ of $X \times \mathbb{N}$ such that
	\begin{enumerate}
		\item $(x,1) \in A_x$ for all $x \in X$;
		\item $\dfrac{|(A_x \setminus A_y)| + |(A_y \setminus A_x)|}{|A_x \cap A_y|} < \varepsilon$ whenever $d(x,y) \leq r$;
		\item there exists $R >0$ such that if $(x,m), (y,m) \in A_z$  for some $z \in X$, then $d(x,y) \leq R$.
	\end{enumerate}
\end{definition}

\begin{theorem}[{\citealt[Theorem 2.2]{MR1728880}}]
\label{thm:yu}
	If a discrete metric space $X$ has property A, then $X$ admits a coarse embedding into a Hilbert space.
\end{theorem}

The following definition was introduced by \citet{enflo1970problem} to answer negatively a question of Smirnov about uniform homeomorphisms into $L_2[0,1]$. Indeed, the negative answer to Gromov's question by Dranishnikov et al. was inspired by Enflo's negative answer to Smirnov's. 

\begin{definition}
	Let $q \geq 0$. A metric space $(X,d)$ has \emph{generalized roundness} $q$ if for any $n \in \mathbb{N}$ and $a_1, \dots, a_n, b_1, \dots, b_n \in X$, we have
	\[
	\sum_{i < j}(d(a_i, a_j)^q + d(b_i, b_j)^q) \leq \sum_{i,j} d(a_i, b_j)^q
	\]
Similarly to negative type, we define the generalized roundness of a metric space $(X, d)$ to be the supremum of the set of $q \in [0, \oo)$ such that $(X,d)$ has generalized roundness $q$.
\end{definition}

\section{Coarse embeddability of diagrams with bottleneck distance}
\label{sec:coarseembeddabilitybottleneck}

The main result of this section is that there does not exist a coarse embedding of $\DGM{\oo}$ into a Hilbert space. This implies that the generalized roundness and asymptotic dimension of $\DGM{\oo}$ are $0$ and $\oo$, respectively. We also show that any separable, bounded metric space has an isometric embedding into the space of persistence diagrams with the bottleneck distance. The isometric embedding in question can be thought of as a shifted version of the Kuratowski embedding.

\begin{figure}
\begin{minipage}{0.45\textwidth}
\begin{tikzpicture}[square/.style={regular polygon, regular polygon sides=4}, triangle/.style={regular polygon, regular polygon sides=3}]
	\draw (0,0) -- node[below] {0.5} (3.5,1);
	\draw (0,0) -- node[left] {$0.8\ $} (3,3);
	\draw (3,3) -- node[right] {0.3} (3.5,1);
	\draw (0,0) node[circle, fill, blue, inner sep = 3pt] {};
	\draw (3.5,1)	node[triangle, fill, green, inner sep = 2.3pt] {};
	\draw (3, 3) node[square, fill, red, inner sep = 2.5pt] {};
	\draw (-.4,0) node {$x_3$};
	\draw (3,3.4) node {$x_1$};
	\draw (4,1) node {$x_2$};
\end{tikzpicture}
\end{minipage}%
\begin{minipage}{0.45\textwidth}
\begin{tikzpicture}[square/.style={regular polygon, regular polygon sides=4}, triangle/.style={regular polygon, regular polygon sides=3}]
	  \draw[->] (-0.5,0) -- (5,0) node[right] {$x$};
      \draw[->] (0,-0.5) -- (0,7) node[above] {$y$};
      \draw[->] (0,0) -- (5,5) node[right] {$\Delta$};
      \draw[{[-]}, ultra thick, black] (0,2) -- (0,3);
      \draw[{[-]}, ultra thick, black] (2,4) -- (2,5);
      \draw[{[-]}, ultra thick, black] (4,6) -- (4,7);
      \draw (-0.6, 2) node {$(0, 2)$};
      \draw (-0.6, 3) node {$(0, 3)$};
      \draw (1.4, 4) node {$(2, 4)$};
      \draw (1.4, 5) node {$(2, 5)$};
      \draw (3.4, 6) node {$(4, 6)$};
      \draw (3.4, 7) node {$(4, 7)$};
      \draw (0,2) node[square, fill, red, inner sep = 1.6pt] {};
      \draw (2,4.3) node[square, fill, red, inner sep = 1.6pt] {};
      \draw (4,6.8) node[square, fill, red, inner sep = 1.6pt] {};
      \draw (2,4) node[triangle, fill, green, inner sep = 1.5pt] {};
      \draw (0,2.3) node[triangle, fill, green, inner sep = 1.5pt] {};
      \draw (4,6.5) node[triangle, fill, green, inner sep = 1.5pt] {};
      \draw (4,6) node[circle, fill, blue, inner sep = 2pt] {};
      \draw (0,2.8) node[circle, fill, blue, inner sep = 2pt] {};
      \draw (2,4.5) node[circle, fill, blue, inner sep = 2pt] {};
\end{tikzpicture}
\end{minipage}
\caption{A metric space with three points and its image in $\DGM{\oo}$ under the map defined in Theorem~\ref{thm:finitems} for $c = 1$, $x_i \mapsto \{(2(k-1), 2k + d(x_i, x_k)\}_{k=1}^{3}$.}
\label{fig:isoembedbottleneck}
\end{figure}
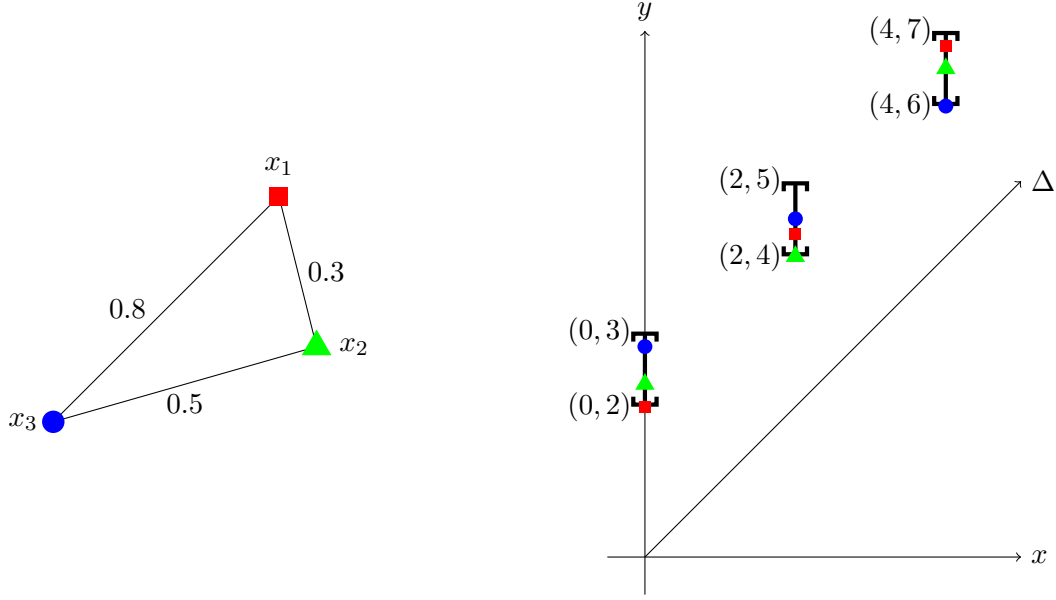

\begin{theorem}
\label{thm:finitems}
Suppose $(X,d)$ is a separable, bounded metric space. Then there exists an isometric embedding $\PHI: (X,d) \to \DGM{\oo}$. Moreover, if $c > \sup \{ d(x,y)\ |\ x,y \in X\}$, we may choose $\PHI$ such that $\PHI(X) \subseteq B(\emptyset, \frac{3c}{2})\setminus B(\emptyset, c)$, where $B(\emptyset, r) = \{D \in \text{Dgm}_{\oo}\ |\ w_{\oo}(D, \emptyset) < r\}$.
\end{theorem}

\begin{proof}

Let $c > \sup \{ d(x,y)\ |\ x,y \in X\}$. Let $\{x_k\}_{k=1}^{\oo}$ be a countable, dense subset of $(X,d)$. Consider the following map.
\begin{align*}
\PHI: (X,d) &\to \DGM{\oo} \\
x &\mapsto \{(2c(k-1), 2ck + d(x, x_k)\}_{k=1}^{\oo}
\end{align*}
Note that for any $x\in X$ and $k \in \mathbb{N}$,
\[
d_\oo((2c(k-1), 2ck + d(x, x_k)), \Delta) = c + \frac{d(x, x_k)}{2} < \frac{3c}{2},
\]
so $\tilde{w}_{\oo}(\PHI(x), \emptyset) < \oo$ for every $x \in X$ and thus $\PHI$ is well-defined. Moreover, since
\[
w_{\oo}(\PHI(x), \emptyset) = \sup_{1 \leq k < \oo} d_\oo((2c(k-1), 2ck + d(x, x_k)), \Delta),
\]
it follows that $\PHI(x) \in B(\emptyset, \frac{3c}{2})\setminus B(\emptyset, c)$.
 A visualization of the image of $\PHI$ for a metric space with three points is shown in Figure \ref{fig:isoembedbottleneck}. We now show that for $y \in X$ an optimal partial matching of $\PHI(x)$ and $\PHI(y)$ matches points in each diagram with the same first coordinate, and the cost of this partial matching is $d(x, y)$. 

For the equivalence classes $\PHI(x)$ and $\PHI(y)$, choose representative persistence diagrams $D_x : \mathbb{N} \to \HP$ and $D_y: \mathbb{N} \to \HP$. Consider the partial matching $(\mathbb{N}, \mathbb{N}, \ID_\mathbb{N})$ between $D_x$ and $D_y$, i.e. $(2c(k-1), 2ck + d(x, x_k))$ is matched with $(2c(k-1), 2ck + d(y, x_k))$ for every $k \in \mathbb{N}$. Observe that $d_\oo(D_x(k), D_y(k)) = |d(x, x_k) - d(y, x_k)|$ for every $k$, so the cost of this partial matching is $\sup_k |d(x, x_k) - d(y, x_k)|$.
By the triangle inequality, 
\[
\sup_k |d(x, x_k) - d(y, x_k)| \leq d(x, y).
\]
Since $\{x_k\}_{k=1}^{\oo}$ is dense, for every $\varepsilon > 0$, there exists a $k$ such that $d(x, x_k) < \varepsilon$, so
\[
|d(x,x_k) - d(y, x_k)| \geq d(y, x_k) - d(x, x_k) \geq d(x,y) - 2d(x, x_k) > d(x,y) - 2\varepsilon. 
\]
This implies that $\sup_k |d(x, x_k) - d(y, x_k)| \geq d(x,y)$ and $\C_\oo(\ID_\mathbb{N}) = d(x,y)$. 

We will now prove that the partial matching described above is optimal. Suppose $I, J \subseteq \mathbb{N}$ and $(I, J, f)$ is a different partial matching between $D_x$ and $D_y$. Then there exists a $k \in \mathbb{N}$ such that either $k \notin I$ or $k \in I$ and $f(k) \neq k$. If $k \notin I$, then
\[
\C_{\oo}(f) \geq d_\oo((2c(k-1), 2ck + d(x, x_k)), \Delta) \geq c.
\]
If $k \in I$ and $f(k) = k' \neq k$, then
\[
\C_{\oo}(f) \geq \|(2c(k-1), 2ck + d(x, x_k)) - (2c(k'-1), 2ck' + d(y, x_{k'}))\|_{\oo} \geq 2c.
\]
Therefore, $\C_{\oo}(f) \geq c > d(x,y)$. Hence, $w_{\oo}(\PHI(x), \PHI(y)) = d(x, y)$, i.e. $\PHI$ is an isometric embedding.
\end{proof}

We now apply Theorem \ref{thm:finitems} to show the generalized roundness of $\DGM{\oo}$ is $0$. To do so, we embed a family of finite metric spaces, whose generalized roundness was observed by \citet{enflo1970problem} to converge to $0$, into $\DGM{\oo}$. One element of this family is shown in Figure~\ref{fig:knn}.

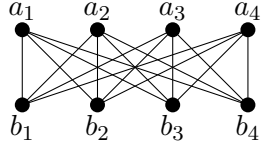
\begin{figure}[h]
\centering
	\begin{tikzpicture}
		\draw (0, 0) node[circle, fill, inner sep = 2pt] {};
		\draw (0, 0) node[below] {$b_1$};
		\draw (1, 0) node[circle, fill, inner sep = 2pt] {};
		\draw (1, 0) node[below] {$b_2$};
		\draw (2, 0) node[circle, fill, inner sep = 2pt] {};
		\draw (2, 0) node[below] {$b_3$};
		\draw (3, 0) node[circle, fill, inner sep = 2pt] {};
		\draw (3, 0) node[below] {$b_4$};
		\draw (0, 1) node[circle, fill, inner sep = 2pt] {};
		\draw (0, 1) node[above] {$a_1$};
		\draw (1, 1) node[circle, fill, inner sep = 2pt] {};
		\draw (1, 1) node[above] {$a_2$};
		\draw (2, 1) node[circle, fill, inner sep = 2pt] {};
		\draw (2, 1) node[above] {$a_3$};
		\draw (3, 1) node[circle, fill, inner sep = 2pt] {};
		\draw (3, 1) node[above] {$a_4$};
		\draw (0,0) -- (0,1);
		\draw (0,0) -- (1,1);
		\draw (0,0) -- (2,1);
		\draw (0,0) -- (3,1);
		\draw (1,0) -- (0,1);
		\draw (1,0) -- (1,1);
		\draw (1,0) -- (2,1);
		\draw (1,0) -- (3,1);
		\draw (2,0) -- (0,1);
		\draw (2,0) -- (1,1);
		\draw (2,0) -- (2,1);
		\draw (2,0) -- (3,1);
		\draw (3,0) -- (0,1);
		\draw (3,0) -- (1,1);
		\draw (3,0) -- (2,1);
		\draw (3,0) -- (3,1);
	\end{tikzpicture}	
	\caption{The metric space obtained from the complete bipartite graph $K_{n,n}$ when $n = 4$.}
	\label{fig:knn}
\end{figure}

\begin{corollary}
\label{cor:grbottleneck}
	The generalized roundness of $\DGM{\oo}$ is zero.
\end{corollary}

\begin{proof}
	Let $n \geq 2$. Define $K_{n,n} = \{a_1, \dots, a_n, b_1, \dots, b_n\}$ and equip this set with the metric $d(a_i, a_j) = d(b_i, b_j) = 2$ for any $i,j \in \{1, \dots, n\}$ with $i \neq j$ and $d(a_i, b_j) = 1$. \citet{enflo1970problem} remarks that $X_n$ has generalized roundness that converges to $0$ as $n\to \oo$. Indeed,
	\begin{align*}
		\sum_{i < j}(d(a_i, a_j)^q + d(b_i, b_j)^q) &\leq \sum_{i,j} d(a_i, b_j)^q \iff \\
		n(n-1)2^q &\leq n^2 \iff \\
		q &\leq \log_2 (1 + (n-1)^{-1}).
	\end{align*}
Hence, $X_n$ has generalized roundness at most $\log_2 (1 + (n-1)^{-1})$ which tends to $0$ as $n$ increases. By Theorem \ref{thm:finitems}, we may isometrically embed $X_n$ into $\DGM{\oo}$ for any $n$ so the generalized roundness of $\DGM{\oo}$ must be zero.
\end{proof}

Our next result is that $\DGM{\oo}$ does not admit a coarse embedding into a Hilbert space. The proof relies on a construction of \citet{dranishnikov2002uniform} based on ideas of \citet{enflo1970problem}.
\begin{theorem}
\label{thm:embeddingdb}
	$\DGM{\oo}$ does not admit a coarse embedding into a Hilbert space.
\end{theorem}

\begin{proof}
	Define $\Z_n$ to be the integers mod $n$ with $d_n$, the metric induced by the standard metric $d(x,y) = |x-y|$ on $\Z$. Define $\Z_n^m$ to be the Cartesian product of $m$ copies of $\Z_n$ with the following metric,
	\[
	d_{n,m}(([k_1], \dots, [k_m]), ([l_1], \dots, [l_m])) = \max_{1 \leq i \leq m} d_n([k_i], [l_i]).
	\]
	Let $X$ be the disjoint union of $\Z_n^m$ for every $n,m \geq 1$ and suppose $\tilde{d}$ is a metric on $X$ satisfying the following.
	\begin{enumerate}[label=(\arabic*)]
		\item The restriction of $\tilde{d}$ to each $\Z_n^m$ coincides with $d_{n, m}$.
		\item $\tilde{d}(x,y) \geq n + m + n' + m'$ if $x \in \Z_{n}^{m}$, $y \in \Z_{n'}^{m'}$, and $(n, m) \neq (n', m')$.
	\end{enumerate}
	
	Proposition 6.3 of \citet{dranishnikov2002uniform} shows that any such $(X,\tilde{d})$ does not admit a coarse embedding into a Hilbert space. Hence, it suffices to construct such an $(X, \tilde{d})$ and an isometric embedding of it into $\DGM{\oo}$, since a coarse embedding of $\DGM{\oo}$ into a Hilbert space would restrict to a coarse embedding of $(X, \tilde{d})$ into a Hilbert space. 
	
	Choose an enumeration $\{(n_i, m_i)\}_{i=1}^{\oo}$ of $\mathbb{N} \times \mathbb{N}$ such that $i < j$ implies $n_i + m_i \leq n_j + m_j$, for instance, $(1,1), (1,2), (2,1), (1,3), (2, 2), (3, 1), (1,4)$, etc. Define $c_1 = 1$ and for $i \geq 2$, $c_i = 4\max(c_{i-1}, n_i + m_i)$. For every $(n_i, m_i)$, note that $c_i > n_i > \max \{d_{n_i, m_i}(x,y)\ |\ x,y \in \Z_{n_i}^{m_i} \}$. So by Theorem \ref{thm:finitems}, there exists an isometry $\PHI_i : \Z_{n_i}^{m_i} \to \DGM{\oo}$ such that $\PHI_i(\Z_{n_i}^{m_i}) \subseteq B(\emptyset, \frac{3c_i}{2}) \setminus B(\emptyset, c_i)$. 
	
	Define $\PHI: X \to \text{Dgm}_\oo$ by $\PHI(x) = \PHI_i(x)$ for $x \in \Z_{n_i}^{m_i}$ and define $\tilde{d}(x,y) = w_\oo(\PHI(x), \PHI(y))$ for any $x, y \in X$. By the definition of $\tilde{d}$, $\PHI: (X, \tilde{d}) \to \DGM{\oo}$ is an isometry. If $x, y \in \Z_{n_i}^{m_i}$, then $\tilde{d}(x,y) = w_\oo(\PHI_i(x), \PHI_i(y)) = d_{n_i, m_i}(x,y)$ so $\tilde{d}$ satisfies $(1)$ above. It only remains to show $\tilde{d}$ satisfies $(2)$.
	
	Suppose $x \in \Z_{n_i}^{m_i}$, $y \in \Z_{n_j}^{m_j}$, and $(n_i, m_i) \neq (n_j, m_j)$. We may assume $i < j$. By construction, $\PHI(x) = \PHI_i(x) \in B(\emptyset, \frac{3c_i}{2}) \setminus B(\emptyset, c_i)$ and $\PHI(y) = \PHI_j(y) \in B(\emptyset, \frac{3c_j}{2}) \setminus B(\emptyset, c_j)$, which implies by the triangle inequality for $w_{\oo}$ that
	\[
	\tilde{d}(x,y) = w_\oo(\PHI(x), \PHI(y)) \geq w_\oo(\PHI(y), \emptyset) - w_\oo(\PHI(x), \emptyset) > c_j - \frac{3c_i}{2}.
	\]
	Additionally, we have $n_i + m_i \leq n_j + m_j$ and $c_j  \geq 4\max(c_i, n_j + m_j) \geq 2(c_i + (n_j+m_j))$, so
	\[
	\tilde{d}(x,y) > c_j - \frac{3c_i}{2} \geq 2(n_j+m_j) + 2c_i - \frac{3c_i}{2} > n_i + m_i + n_j + m_j.
	\]
	We have shown that $\tilde{d}$ satisfies $(2)$ which completes the proof.
\end{proof}

\begin{remark} \label{rem:finite}
  For a finite metric space, the isometric embedding defined in Theorem \ref{thm:finitems} sends each point to a persistence diagram of finite cardinality in $\DGM{\oo}$. In particular, the map $\PHI_i: \Z^{m_i}_{n_i} \to \DGM{\oo}$ given in the proof of Theorem \ref{thm:finitems} has an image consisting of finite persistence diagrams. Since $X$ is the disjoint union of $\Z^m_n$ for every $n,m \geq 1$, it follows that $\PHI: (X, \tilde{d}) \to \DGM{\oo}$ sends each point in the metric space $X$ to a finite persistence diagram. Hence, the proof of Theorem \ref{thm:embeddingdb} gives the slightly stronger result that the space of finite persistence diagrams with the bottleneck distance does not admit a coarse embedding into a Hilbert space.
\end{remark}

Theorem \ref{thm:embeddingdb} and Remark~\ref{rem:finite} give the impossibility of coarsely embedding the space of finite persistence diagrams with the bottleneck distance into a Hilbert space. The primary motivation for this result was the application of kernel methods to persistent homology. In computational settings, the persistence diagrams of interest are frequently the result of applying homology to a filtered finite simplicial complex. We refer the interested reader to \citet{oudot:book}. Hence, one may ask whether this more restricted space of persistence diagrams, i.e. the subspace arising from homology of filtered finite simplicial complexes, admits a coarse embedding into a Hilbert space. Unfortunately, this is easily seen to be false by the following.

\begin{lemma} \label{lem:finite}
	Every finite persistence diagram is realizable as the persistent homology of a filtered finite simplicial complex.
\end{lemma}
      
\begin{proof}
  Suppose $D: \{1, \dots, n \} \to \HP$ is a persistence diagram and define $(x_i, y_i) = D(i)$. Let $V$ be the set $\{a_i, b_i, c_i\}_{i = 1}^n$. Consider the simplicial complex on $V$ that is the disjoint union of $n$ $2$-simplices and has the filtration given by assigning the value $x_i$ to
  $\{a_i\}, \{b_i\}, \{c_i\}, \{a_i, b_i\}, \{a_i, c_i\}, \{b_i, c_i\}$ and the value $y_i$ to 
  $\{a_i, b_i, c_i\}$. 
  Applying the simplicial homology functor $H_1(-, \Z_2)$ recovers the persistence diagram $D$. To see this, note that the persistent homology of the filtration on $V$ is the direct sum of the persistent homology of the filtration on each individual triangle since the triangles are mutually disjoint. For triangle $i$, the $1$-skeleton appears at time $x_i$ giving rise to one degree-$1$ homology generator that vanishes when the $2$-simplex is added at time $y_i$. 
\end{proof}

Property A is a concrete condition satisfiable by a discrete metric space that implies it can be coarsely embedded into a Hilbert space. In the following proposition, we show that a semi-metric space being of $q$-negative type for some positive $q$ is similarly a concrete condition that implies coarse embeddability into a Hilbert space.

\begin{proposition}
\label{prop:posnegtype}
 A semi-metric space $(X,d)$ of $q$-negative type for some $q > 0$ admits a coarse embedding into a Hilbert space.	
\end{proposition}

\begin{proof}
Suppose there exists a $q > 0$ such that the metric space $(X,d)$ has $q$-negative type. Define $f(t) = t^{q/2}$ and observe that $fd(x,x) = 0^{q/2} = 0$ and $fd(x,y) = fd(y,x)$ so $(X,fd)$ is a semi-metric space. Let $x_1, \dots, x_n \in X$ and $a_1, \dots, a_n \in \R$ such that $\sum_{i=1}^n a_i = 0$. Then
\[
	\sum_{i,j =1}^{n} a_ia_j(fd(x_i,x_j))^2 =  \sum_{i,j =1}^{n} a_ia_jd(x_i,x_j)^q\leq 0,
\]
so $(X,fd)$ is a semi-metric space of $2$-negative type. By Theorem \ref{thm:negtypeiso}, there exists an isometric embedding $\PHI$ from $(X,fd)$ into a Hilbert space $\mathcal{H}$. Define $\rho_+ = \rho_- = f$. It follows that $\PHI$ satisfies the requirements of a coarse embedding of $(X,d)$ into $\mathcal{H}$, i.e.
\[
\rho_+(d(x,y)) = \rho_-(d(x,y)) = fd(x,y) = \| \PHI(x) - \PHI(y) \|_{\mathcal{H}}. \qedhere 
\]
\end{proof}

\begin{remark} \label{rem:negative-type}
Since $\DGM{\oo}$ does not admit a coarse embedding into a Hilbert space, Proposition \ref{prop:posnegtype} implies that $\DGM{\oo}$ is of $0$-negative type. This also follows from Corollary \ref{cor:grbottleneck} and the result of \citet{lennard1997generalized} that a space has generalized roundness $q$ iff it has $q$-negative type. Finally, we state two corollaries of Theorem \ref{thm:embeddingdb} that answer Questions 3.10 and 3.11 of \citet{2019arXiv190202288B}.
\end{remark}

\begin{corollary} \label{cor:propertya}
	$\DGM{\oo}$ contains a discrete subspace that fails to have property A.
\end{corollary}

\begin{proof}
In the proof of Theorem \ref{thm:embeddingdb}, we consider a discrete metric space $(X, \tilde{d})$ and prove it embeds in $\DGM{\oo}$ via an isometric embedding $\PHI$. \citet{dranishnikov2002uniform} have shown that $(X, \tilde{d})$ does not admit a coarse embedding into a Hilbert space so by Theorem \ref{thm:yu}, $\PHI(X)$ fails to have property A.
\end{proof}

\begin{corollary} \label{cor:asymptotic-dimension}
	$\DGM{\oo}$ has infinite asymptotic dimension.
\end{corollary}

\begin{proof}
If $\DGM{\oo}$ had finite asymptotic dimension, then it would admit a coarse embedding into a Hilbert space by Theorem \ref{thm:roe}, which contradicts Theorem \ref{thm:embeddingdb}.
\end{proof}


\begin{thebibliography}{21}
\providecommand{\natexlab}[1]{#1}
\providecommand{\url}[1]{\texttt{#1}}
\expandafter\ifx\csname urlstyle\endcsname\relax
  \providecommand{\doi}[1]{doi: #1}\else
  \providecommand{\doi}{doi: \begingroup \urlstyle{rm}\Url}\fi

\bibitem[{Bell} et~al.(2019){Bell}, {Lawson}, {Pritchard}, and
  {Yasaki}]{2019arXiv190202288B}
Greg {Bell}, Austin {Lawson}, C.~Neil {Pritchard}, and Dan {Yasaki}.
\newblock {The space of persistence diagrams fails to have Yu's property A}.
\newblock \emph{arXiv e-prints}, art. arXiv:1902.02288, Feb 2019.

\bibitem[Berg et~al.(1984)Berg, Christensen, and Ressel]{MR747302}
Christian Berg, Jens Peter~Reus Christensen, and Paul Ressel.
\newblock \emph{Harmonic analysis on semigroups: Theory of positive definite
  and related functions}, volume 100 of \emph{Graduate Texts in Mathematics}.
\newblock Springer-Verlag, New York, 1984.
\newblock ISBN 0-387-90925-7.
\newblock \doi{10.1007/978-1-4612-1128-0}.
\newblock URL \url{https://doi.org/10.1007/978-1-4612-1128-0}.

\bibitem[Blumberg et~al.(2014)Blumberg, Gal, Mandell, and Pancia]{MR3230014}
Andrew~J. Blumberg, Itamar Gal, Michael~A. Mandell, and Matthew Pancia.
\newblock Robust statistics, hypothesis testing, and confidence intervals for
  persistent homology on metric measure spaces.
\newblock \emph{Found. Comput. Math.}, 14\penalty0 (4):\penalty0 745--789,
  2014.
\newblock ISSN 1615-3375.
\newblock \doi{10.1007/s10208-014-9201-4}.
\newblock URL \url{https://doi.org/10.1007/s10208-014-9201-4}.

\bibitem[Bubenik and Vergili(2018)]{MR3927353}
Peter Bubenik and Tane Vergili.
\newblock Topological spaces of persistence modules and their properties.
\newblock \emph{J. Appl. Comput. Topol.}, 2\penalty0 (3-4):\penalty0 233--269,
  2018.
\newblock ISSN 2367-1726.
\newblock \doi{10.1007/s41468-018-0022-4}.
\newblock URL \url{https://mathscinet.ams.org/mathscinet-getitem?mr=3927353}.

\bibitem[Carri{\`e}re and Bauer(2019)]{MR3968607}
Mathieu Carri{\`e}re and Ulrich Bauer.
\newblock On the metric distortion of embedding persistence diagrams into
  separable {H}ilbert spaces.
\newblock In \emph{35th {I}nternational {S}ymposium on {C}omputational
  {G}eometry}, volume 129 of \emph{LIPIcs. Leibniz Int. Proc. Inform.}, pages
  Art. No. 21, 15. Schloss Dagstuhl. Leibniz-Zent. Inform., Wadern, 2019.
\newblock URL \url{https://mathscinet.ams.org/mathscinet-getitem?mr=3968607}.

\bibitem[Chazal et~al.(2009)Chazal, Cohen-Steiner, Glisse, Guibas, and
  Oudot]{Chazal:2009:PPM:1542362.1542407}
Fr{\'e}d{\'e}ric Chazal, David Cohen-Steiner, Marc Glisse, Leonidas~J. Guibas,
  and Steve~Y. Oudot.
\newblock Proximity of persistence modules and their diagrams.
\newblock In \emph{Proceedings of the Twenty-fifth Annual Symposium on
  Computational Geometry}, SCG '09, pages 237--246, New York, NY, USA, 2009.
  ACM.
\newblock ISBN 978-1-60558-501-7.
\newblock \doi{10.1145/1542362.1542407}.
\newblock URL \url{http://doi.acm.org/10.1145/1542362.1542407}.

\bibitem[Cohen-Steiner et~al.(2007)Cohen-Steiner, Edelsbrunner, and
  Harer]{MR2279866}
David Cohen-Steiner, Herbert Edelsbrunner, and John Harer.
\newblock Stability of persistence diagrams.
\newblock \emph{Discrete Comput. Geom.}, 37\penalty0 (1):\penalty0 103--120,
  2007.
\newblock ISSN 0179-5376.
\newblock \doi{10.1007/s00454-006-1276-5}.
\newblock URL \url{https://doi.org/10.1007/s00454-006-1276-5}.

\bibitem[Cohen-Steiner et~al.(2010)Cohen-Steiner, Edelsbrunner, Harer, and
  Mileyko]{MR2594441}
David Cohen-Steiner, Herbert Edelsbrunner, John Harer, and Yuriy Mileyko.
\newblock Lipschitz functions have {$L_p$}-stable persistence.
\newblock \emph{Found. Comput. Math.}, 10\penalty0 (2):\penalty0 127--139,
  2010.
\newblock ISSN 1615-3375.
\newblock \doi{10.1007/s10208-010-9060-6}.
\newblock URL \url{https://mathscinet.ams.org/mathscinet-getitem?mr=2594441}.

\bibitem[Dranishnikov et~al.(2002)Dranishnikov, Gong, Lafforgue, and
  Yu]{dranishnikov2002uniform}
A.~N. Dranishnikov, G.~Gong, V.~Lafforgue, and G.~Yu.
\newblock Uniform embeddings into {H}ilbert space and a question of {G}romov.
\newblock \emph{Canad. Math. Bull.}, 45\penalty0 (1):\penalty0 60--70, 2002.
\newblock ISSN 0008-4395.
\newblock \doi{10.4153/CMB-2002-006-9}.
\newblock URL \url{https://doi.org/10.4153/CMB-2002-006-9}.

\bibitem[Enflo(1969)]{enflo1970problem}
Per Enflo.
\newblock On a problem of {S}mirnov.
\newblock \emph{Ark. Mat.}, 8:\penalty0 107--109, 1969.
\newblock ISSN 0004-2080.
\newblock \doi{10.1007/BF02589550}.
\newblock URL \url{https://doi.org/10.1007/BF02589550}.

\bibitem[Gromov(1993)]{gromov1993asymptotic}
M.~Gromov.
\newblock Asymptotic invariants of infinite groups.
\newblock In \emph{Geometric group theory, {V}ol. 2 ({S}ussex, 1991)}, volume
  182 of \emph{London Math. Soc. Lecture Note Ser.}, pages 1--295. Cambridge
  Univ. Press, Cambridge, 1993.

\bibitem[Lennard et~al.(1997)Lennard, Tonge, and
  Weston]{lennard1997generalized}
C.~J. Lennard, A.~M. Tonge, and A.~Weston.
\newblock Generalized roundness and negative type.
\newblock \emph{Michigan Math. J.}, 44\penalty0 (1):\penalty0 37--45, 1997.
\newblock ISSN 0026-2285.
\newblock \doi{10.1307/mmj/1029005619}.
\newblock URL \url{https://doi.org/10.1307/mmj/1029005619}.

\bibitem[Mileyko et~al.(2011)Mileyko, Mukherjee, and
  Harer]{mileyko2011probability}
Yuriy Mileyko, Sayan Mukherjee, and John Harer.
\newblock Probability measures on the space of persistence diagrams.
\newblock \emph{Inverse Problems}, 27\penalty0 (12):\penalty0 124007, 22, 2011.
\newblock ISSN 0266-5611.
\newblock \doi{10.1088/0266-5611/27/12/124007}.
\newblock URL \url{https://doi.org/10.1088/0266-5611/27/12/124007}.

\bibitem[Oudot(2015)]{oudot:book}
Steve~Y. Oudot.
\newblock \emph{Persistence theory: from quiver representations to data
  analysis}, volume 209 of \emph{Mathematical Surveys and Monographs}.
\newblock American Mathematical Society, Providence, RI, 2015.
\newblock ISBN 978-1-4704-2545-6.
\newblock \doi{10.1090/surv/209}.
\newblock URL \url{https://mathscinet.ams.org/mathscinet-getitem?mr=3408277}.

\bibitem[Roe(2003)]{roe2003lectures}
John Roe.
\newblock \emph{Lectures on coarse geometry}, volume~31 of \emph{University
  Lecture Series}.
\newblock American Mathematical Society, Providence, RI, 2003.
\newblock ISBN 0-8218-3332-4.
\newblock \doi{10.1090/ulect/031}.
\newblock URL \url{https://doi.org/10.1090/ulect/031}.

\bibitem[Schoenberg(1935)]{MR1503248}
I.~J. Schoenberg.
\newblock Remarks to {M}aurice {F}r\'{e}chet's article ``{S}ur la
  d\'{e}finition axiomatique d'une classe d'espace distanci\'{e}s
  vectoriellement applicable sur l'espace de {H}ilbert'' [{MR}1503246].
\newblock \emph{Ann. of Math. (2)}, 36\penalty0 (3):\penalty0 724--732, 1935.
\newblock ISSN 0003-486X.
\newblock \doi{10.2307/1968654}.
\newblock URL \url{https://doi.org/10.2307/1968654}.

\bibitem[Schoenberg(1938)]{MR1501980}
I.~J. Schoenberg.
\newblock Metric spaces and positive definite functions.
\newblock \emph{Trans. Amer. Math. Soc.}, 44\penalty0 (3):\penalty0 522--536,
  1938.
\newblock ISSN 0002-9947.
\newblock \doi{10.2307/1989894}.
\newblock URL \url{https://doi.org/10.2307/1989894}.

\bibitem[Steinwart and Christmann(2008)]{MR2450103}
Ingo Steinwart and Andreas Christmann.
\newblock \emph{Support vector machines}.
\newblock Information Science and Statistics. Springer, New York, 2008.
\newblock ISBN 978-0-387-77241-7.

\bibitem[{Turner} and {Spreemann}(2019)]{2019arXiv190301051T}
Katharine {Turner} and Gard {Spreemann}.
\newblock {Same but Different: distance correlations between topological
  summaries}.
\newblock \emph{arXiv e-prints}, art. arXiv:1903.01051, Mar 2019.

\bibitem[Wells and Williams(1975)]{wells2012embeddings}
J.~H. Wells and L.~R. Williams.
\newblock \emph{Embeddings and extensions in analysis}.
\newblock Springer-Verlag, New York-Heidelberg, 1975.
\newblock Ergebnisse der Mathematik und ihrer Grenzgebiete, Band 84.

\bibitem[Yu(2000)]{MR1728880}
Guoliang Yu.
\newblock The coarse {B}aum-{C}onnes conjecture for spaces which admit a
  uniform embedding into {H}ilbert space.
\newblock \emph{Invent. Math.}, 139\penalty0 (1):\penalty0 201--240, 2000.
\newblock ISSN 0020-9910.
\newblock \doi{10.1007/s002229900032}.
\newblock URL \url{https://doi.org/10.1007/s002229900032}.

\end{thebibliography}
\end{document}